\documentclass[letterpaper]{article}
\usepackage{proceed2e}
\usepackage[margin=1in]{geometry}

\usepackage{times}
\usepackage{amsmath}
\usepackage{amsfonts}
\usepackage{amssymb}
\usepackage{amsthm}
\usepackage{color}
\usepackage{graphicx}
\usepackage{subfigure}
\usepackage[]{natbib}
\usepackage{booktabs} 
\usepackage{multirow}
\usepackage{arydshln} 
\usepackage{url}

\newcommand{\mb}[1]{\mathbf{#1}}

\newcommand{\bs}[1]{\boldsymbol{#1}}

\newtheorem{theorem}{Theorem}

\title{Sylvester Normalizing Flows for Variational Inference}

\author{ {\bf Rianne van den Berg\thanks{\hspace{2mm}Equal contribution.}}\\
University of Amsterdam \\
\And
{\bf Leonard Hasenclever$^{*}$}\\
University of Oxford\\
\And
{\bf Jakub M. Tomczak}\\
University of Amsterdam\\
\And
{\bf Max Welling}\\
University of Amsterdam\\
}

\begin{document}

\maketitle

\begin{abstract}
Variational inference relies on flexible approximate posterior distributions. Normalizing flows provide a general recipe to construct flexible variational posteriors. We introduce Sylvester normalizing flows, which can be seen as a generalization of planar flows. Sylvester normalizing flows remove the well-known single-unit bottleneck from planar flows, making a single transformation much more flexible. We compare the performance of Sylvester normalizing flows against planar flows and inverse autoregressive flows and demonstrate that they compare favorably on several datasets. 
The code of our model is publicly available at \url{https://github.com/riannevdberg/sylvester-flows}.
\end{abstract}

\section{INTRODUCTION}

Stochastic variational inference \citep{Hoffman13a} allows for posterior inference in increasingly large and complex problems using stochastic gradient ascent. In continuous latent variable models, variational inference can be made particularly efficient through the amortized inference, in which inference networks amortize the cost of calculating the variational posterior for a data point \citep{GG:14}. A particularly successful class of models is the variational autoencoder (VAE) in which both the generative model and the inference network are given by neural networks, and sampling from the variational posterior is efficient through the non-centered parameterization \citep{KW:14}, also known as the reparameterization trick \citep{KW:13,RMW:14}.

Despite its success, variational inference has drawbacks compared to other inference methods such as MCMC. Variational inference searches for the best posterior approximation within a parametric family of distributions. Hence, the true posterior distribution can only be recovered exactly if it happens to be in the chosen family. In particular, with widely used simple variational families such as diagonal covariance Gaussian distributions, the variational approximation is likely to be insufficient. More complex variational families enable better posterior approximations, resulting in improved model performance. Therefore, designing tractable and more expressive variational families is an important problem in variational inference \citep{NHS:16, SKW:15, TRB:15}.

\cite{RezMoh2015} introduced a general framework for constructing more flexible variational distributions, called normalizing flows. Normalizing flows transform a base density through a number of invertible parametric transformations with tractable Jacobians into more complicated distributions. They proposed two classes of normalizing flows: planar flows and radial flows. While effective for small problems, these can be hard to train and often many transformations are required to get good performance. For planar flows, \cite{KinSalJoz2016} argue that this is due to the fact that the transformation used acts as a bottleneck, warping one direction at a time. Having a large number of flows makes the inference network very deep and harder to train, empirically resulting in suboptimal performance. \cite{KinSalJoz2016} proposed inverse auto-regressive flows (IAF), achieving state of the art results on dynamically binarized MNIST at the time of publication. 
While very successful, each transformation in IAF only depends on the datapoint $\mb x$ through a context vector, with flow parameters that are independent of the datapoint. 

\textbf{Paper contribution} In this paper, we use Sylvester's determinant identity to introduce Sylvester normalizing flows (SNFs). This family of flows is a generalization of planar flows, removing the bottleneck. We compare a number of different variants of SNFs and show that they compare favorably against planar flows and IAFs. We show that one specific variant of SNF is related to IAF, with the main difference being the amortization strategy of the flow parameters. Besides the usual requirement of having flexible transformations, this demonstrates the importance of having data-dependent flow parameters. Note that this concept generalizes to applying normalizing flows to any conditional distribution, in the sense that the transformation parameters should be functions of the conditioning variable. 

\section{VARIATIONAL INFERENCE}
Consider a probabilistic model with observations $\mb{x}$ and continuous latent variables $\mb{z}$ and model parameters $\theta$. In generative modeling we are often interested in performing maximum (marginal) likelihood learning of the parameters $\mb{\theta}$ of the latent-variable model $p_\mb{\theta}(\mb{x}, \mb{z})$. This requires marginalization over the unobserved latent variables $\mb{z}$. Unfortunately, this integration is generally intractable. Variational inference \citep{Jordan1999} instead introduces a variational approximation $q(\mb{z}|\mb{x})$ to the posterior, to construct a lower bound on the log marginal likelihood:
\begin{align}
\log p_{\theta}(\mb{x}) &\geq \log p_{\theta}(\mb{x}) - \text{KL}(q(\mb{z}|\mb{x}) \,||\, p_{\theta}(\mb{z}|\mb{x}))\label{eq:kl}\\
&= \mathbb{E}_{q}[\log p_\theta(\mb{x}|\mb{z})] - \text{KL}(q(\mb{z}|\mb{x}) \,||\, p(\mb{z})) \label{eq:elbo}\\
&=: -\mathcal{F}(\theta)
\end{align}
This bound is known as the evidence lower bound (ELBO) and $\mathcal{F}$ is referred to as the variational free energy. In equation (\ref{eq:elbo}), the first term represents the reconstruction error, and the second term is the Kullback-Leibler (KL) divergence from the approximate posterior to the prior distribution, which acts as a regularizer. In this paper we consider variational autoencoders (VAEs), where both $p_\theta(\mb{x}|\mb{z})$ and $q(\mb{z}|\mb{x})$ are distributions whose parameters are given by neural networks. That is, we perform amortized inference such that $q(\mb z|\mb x) = q_\phi(\mb z|\mb x)$ with $\phi$ the parameters of the encoder neural network. The parameters $\theta$ and $\phi$  of the generative model and inference model, respectively, are trained jointly through stochastic minimization of $\mathcal{F}(\theta, \phi)$, which can be made efficient through the reparameterization trick \citep{KW:13,RMW:14}.

From equation \eqref{eq:kl} we see that the better the variational approximation to the posterior the tighter the ELBO. 
The simplest, but probably most widely used choice of variation distribution $q_\phi(\mb{z}|\mb{x})$ is diagonal-covariance Gaussians of the form $\mathcal{N}(\mb z|\pmb{\mu}_{\phi}(\mb{x}), ~\pmb{\sigma}_{\phi}^2(\mb{x}))$

However, with such simple variational distributions the ELBO will be fairly loose, resulting in biased maximum likelihood estimates of the model parameters $\mb{\theta}$ (see Fig. \ref{fig:bias}) and harming generative performance. Thus, for variational inference to work well, more flexible approximate posterior distributions are needed.
\begin{figure}
\centering
\includegraphics[width=0.9\linewidth]{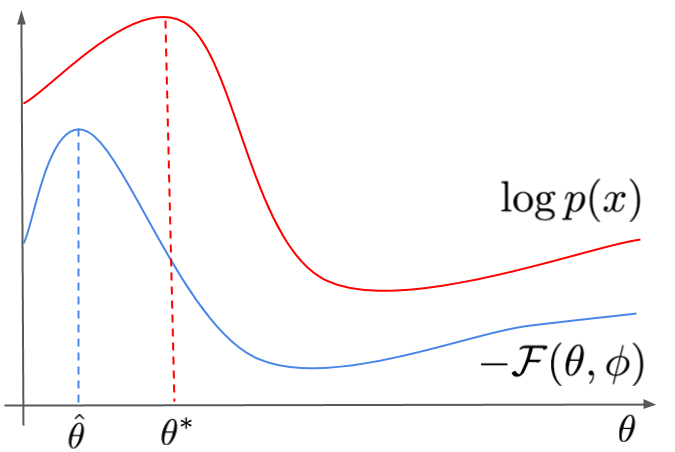}
\caption{Since the ELBO is only a lower bound on the log marginal likelihood, they do not share the same local maxima. The looser the ELBO is the more this can bias maximum likelihood estimates of the model parameters.}
\vspace{-0.4cm}
\label{fig:bias}
\end{figure}

\subsection{NORMALIZING FLOWS} 
\cite{RezMoh2015} propose a way to construct more flexible posteriors by transforming a simple base distribution with a series of invertible transformations (known as normalizing flows) with easily computable Jacobians. The resulting transformed density after one such transformation $f$ is as follows \citep{TT:13,TV:10}:
\begin{align}
p_{1}(\mb{z'}) = p_{0}(\mb{z}) \left| \det \left( \frac{\partial f(\mb{z})}{\partial \mb{z}} \right) \right|^{-1} ,
\end{align}
where $\mb{z'} = f(\mb{z})$, $\mb z, \mb z' \in \mathbb R^D$ and $f: \mathbb R^D \mapsto\mathbb R^D$ is an invertible function. 
In general the cost of computing the Jacobian will be $\mathcal{O}(D^3)$. However, it is possible to design transformations with more efficiently computable Jacobians.

This strategy is used in variational inference as follows: first, a stochastic variable is drawn from a simple base posterior distribution such as a diagonal Gaussian $\mathcal N(\mb z_0|\bs \mu(\mb{x}), \bs \sigma^2(\mb{x}))$. The sample is then transformed  with a number of flows. 
After applying $K$ flows, the final latent stochastic variables are given by
$\mb z_K = f_K \circ\dots f_2 \circ f_1(\mb z_0)$. The corresponding log-density is then given by:
\begin{align}
\log q_K(\mb{z}_K|\mb{x}) =& \log q_0(\mb{z}_0|\mb{x}) \nonumber\\
- \sum_{k=1}^{K}& {\log \left| \det \left( \frac{\partial f_k(\mb{z}_{k-1}; \lambda_k(\mb{x}))}{\partial \mb{z}_{k-1}} \right) \right| },
\end{align}
where $\lambda_k$ are the parameters of the $k$-th transformation. Note that in order to achieve a flexible amortization strategy, the flow parameters $\lambda_k$ can be made dependent on the input data: $\lambda_k = \lambda_k (\mb x)$  \citep{RezMoh2015}. 
Given a variational posterior $q_\phi(\mb{z}|\mb{x})= q_K(\mb{z}|\mb{x})$ parametrized by a normalizing flow of length K, the variational objective can be rewritten as:
\begin{align}
\mathcal{F}(\theta, \phi) &= \mathbb{E}_{q_\phi}[\log q_\phi(\mb{z}|\mb{x})-\log p_\theta(\mb{x}, \mb{z})]\\
&= \mathbb{E}_{q_0}[\log q_0(\mb{z}_0|\mb{x}) - \log p_\theta(\mb{x}, \mb{z}) ]\nonumber\\
&-\mathbb{E}_{q_0}\left[\sum_{k=1}^{K}{\log \left| \det \left( \frac{\partial f_k(\mb{z}_{k-1}; \lambda_k(\mb{x}))}{\partial \mb{z}_{k-1}} \right) \right|}\right].
\end{align}

Normalizing flows are frequently applied to amortized variational inference. Instead of learning the parameters of the posterior distribution for each data point, (such as $\bs \mu$ and $\bs \sigma$ for a Gaussian posterior), the input-dependence of the posterior distribution parameters is modeled through an encoder/inference network. When performing amortized inference for normalizing flows, the flow parameters determine the final distribution, and should thus also be considered functions of the datapoint $\mb x$. This can be achieved through the use of hypernetworks \citep{ha2016hypernetworks}.

\cite{RezMoh2015} introduced a normalizing flow, called planar flow, for which the Jacobian determinant could be computed efficiently. A single transformation of the planar flow is given by:
\begin{align}
\mb{z'} = \mb{z} + \mb{u}h(\mb{w}^T\mb{z} + b).
\label{eq:planar_flow}
\end{align}
Here, $\mb{u},\mb{w}\in \mathbb{R}^D$, $b\in\mathbb{R}$ and $h$ is a suitable smooth activation function. \cite{RezMoh2015} show that for $h=\tanh$, transformations of this kind are invertible as long as $\mb{u}^T\mb{w}\geq -1$. 

By the \textit{Matrix determinant lemma} the Jacobian of this transformation is given by:
\begin{align}
\det\frac{\partial \mb{z'}}{\partial \mb{z}} &= \det\left(\mb{I} + \mb{u}h'(\mb{w}^T\mb{z}+b)\mb{w}^T \right) \notag \\
&= 1 + \mb{u}^Th'(\mb{w}^T\mb{z}+b)\mb{w},
\end{align}
where $h'$ denotes the derivative of $h$ and which can be computed in $O(D)$ time.

In practice, many planar flow transformations are required to transform a simple base distribution into a flexible distribution, especially for high dimensional latent spaces. 
\cite{KinSalJoz2016} argue that this is related to the term $\mb{u}h(\mb{w}^T\mb{z} + b)$ in Eq. \eqref{eq:planar_flow}, which effectively acts as a single-neuron MLP. In the next section we will derive a generalization of planar flows, which does not have a single-neuron bottleneck, while still maintaining the property of an efficiently computable Jacobian determinant.

\section{SYLVESTER NORMALIZING FLOWS}

Consider the following more general transformation similar to a single layer MLP with $M$ hidden units and a residual connection:
\begin{equation}
\mb{z'} = \mb{z} + \mb{A}h(\mb{B}\mb{z} + \mb{b}),
\label{eq:general_planar}
\end{equation}
with $\mb{A}\in\mathbb{R}^{D\times M},\mb{B}\in \mathbb{R}^{M\times D}$, $\mb{b} \in \mathbb{R}^{M}$, and $M \leq D$. The Jacobian determinant of this transformation can be obtained using Sylvester's determinant identity, which is a generalization of the matrix determinant lemma.
\begin{theorem}[Sylvester's determinant identity]\label{thm:Sylvester}
For all $\mb{A}\in\mathbb{R}^{D\times M},\mb{B}\in \mathbb{R}^{M\times D}$, 
\begin{align}
\det\left( {\mb{I}_D + \mb{A}\mb{B}}\right) =\det\left(\mb{I}_M + \mb{B}\mb{A}\right),
\end{align}
where $\mb I_M$ and $\mb I_D$ are $M$ and $D$-dimensional identity matrices, respectively.
\end{theorem}
When $M<D$, the computation of the determinant of a $D\times D$ matrix is thus reduced to the computation of the determinant of an $M \times M$ matrix. 

Using Sylvester's determinant identity, the Jacobian determinant of the transformation in Eq. \eqref{eq:general_planar} is given by:
\begin{equation}
\det \left(\frac{\partial \mb{z'}}{\partial \mb{z}}\right) = \det \left( \mb{I}_M + \mathrm{diag}\left(h'(\mb{B}\mb{z} + \mb{b})\right)\mb{B}\mb{A} \right).
\end{equation}
Since Sylvester's determinant identity plays a crucial role in the proposed family of normalizing flows, we will refer to them as \textit{Sylvester normalizing flows}.

\subsection{PARAMETERIZATION OF $\mathbf{A}$ AND $\mathbf{B}$}

In general, the transformation in (\ref{eq:general_planar}) will not be invertible.
Therefore, we propose the following special case of the above transformation:
\begin{align}
\mb{z'} = \mb{z} + \mb{Q}\mb{R}h(\mb{\tilde R}\mb{Q}^T\mb{z} + \mb{b}) = \phi(\mb{z})\label{eq:orthoflow},
\end{align}
where $\mb{R}$ and $\mb{\tilde R}$ are upper triangular $M \times M$ matrices, and 
\[\mb{Q}=\left( \mb{q}_1 \ldots \mb{q}_M \right)\] with the columns $\mb{q}_{m} \in \mathbb R^D$ forming an orthonormal set of vectors.
By theorem \ref{thm:Sylvester}, the determinant of the Jacobian $\mb{J}$ of this transformation reduces to:
\begin{align}
\det \mb{J}
 &= \det \left( \mb{I}_M + \mathrm{diag}\left(h'(\mb{\tilde R}\mb{Q}^T\mb{z} + \mb{b})\right)\mb{\tilde R}\mb{Q}^T\mb{Q}\mb{R} \right)\notag \\
 &= \det \left(\mb{I}_M + \mathrm{diag}\left(h'(\mb{\tilde R}\mb{Q}^T\mb{z} + \mb{b})\right)\mb{\tilde R}\mb{R}\right),
\end{align}
which can be computed in $O(M)$, since $\mb{\tilde{R}R}$ is also upper triangular. The following theorem gives a sufficient condition for this transformation to be invertible. 

\begin{theorem}
Let $\mb{R}$ and $\mb{\tilde R}$ be upper triangular matrices. Let $h: \mathbb{R}\longrightarrow\mathbb{R}$ be a smooth function with bounded, positive derivative. Then, if the diagonal entries of $\mb{R}$ and $\mb{\tilde R}$ satisfy $r_{ii} \tilde r_{ii} > -1/\|h'\|_\infty$ and $\mb{\tilde R}$ is invertible, the transformation given by \eqref{eq:orthoflow} is invertible. 
\end{theorem}
\begin{proof}
\textbf{Case 1: $\mb{R}$ and $\mb{\tilde{R}}$ diagonal}\par
Recall that one-dimensional real functions with strictly positive derivatives are invertible. The columns of $\mb{Q}$ are orthonormal and span a subspace $\mathcal{W}=\text{span}\{\mb{q}_1,\ldots,\mb{q}_M\}$ of  $\mathbb{R}^D$. Let $\mathcal{W}^{\perp}$ denote its orthogonal complement. We can decompose $\mb{z} =  \mb{z}_\parallel + \mb{z}_\perp$, where $\mb{z}_\parallel \in \mathcal{W}$ and $\mb{z}_\perp\in\mathcal{W}^\perp$. Similarly we can decompose $\mb{z}' =  \mb{z}'_\parallel + \mb{z}'_\perp$. Clearly, $\mb{Q}\mb{R}h(\mb{\tilde R}\mb{Q}^T\mb{z} + \mb b)\in \mathcal{W}$. Hence $\phi$ only acts on $\mb{z}_\parallel$ and $\mb z_\perp = \phi(\mb{z})_\perp = \mb z'_\perp$. Thus, it suffices to consider the effect of $\phi$ on $\mb{z}_\parallel$. Multiplying \eqref{eq:orthoflow} by $\mb{Q}^T$ from the left gives:
\begin{align}
\underbrace{\mb{Q}^T\mb{z}'}_{\mb{v}'} &= \underbrace{\mb{Q}^T\mb{z}}_{\mb{v}} + \mb{R}h(\mb{\tilde R}\underbrace{\mb{Q}^T\mb{z}}_{\mb{v}}+\mb{b}) \notag \\
&= (f_1(v_1), \ldots, f_M(v_M))^T, \label{eq:uncoupled}
\end{align}
where the vectors $\mb{v}$ and $\mb{v}'$ are the respective coordinates of $\mb{z}_\parallel$ and $\mb{z}'_\parallel$ w.r.t. $\mb{q}_1,\ldots,\mb{q}_M$. The dimensions in \eqref{eq:uncoupled} are completely independent and each dimension is transformed by a real function $f_i(v) = v + r_{ii}h(\tilde r_{ii} v + b_i)$. Consider a single dimension $i$ of \eqref{eq:uncoupled}. Since $\|h'\|_\infty r_{ii} \tilde r_{ii} > -1$, we have $f_i'(v)>0$ and thus $f_i$ is invertible. Since all dimensions are independent and the transformation is invertible in each dimension we can find $f^{-1}:\mathcal{W}\rightarrow\mathcal{W}$ such that $\mb{z}_\parallel=f^{-1}(\mb{z}'_\parallel)$. 
Hence we can write the inverse of $\phi$ as:
\begin{align}
\phi^{-1}(\mb{z}') = \underbrace{\mb{z}'_\perp}_{\mb{z}_\perp} + \underbrace{f^{-1}(\mb{z}'_\parallel)}_{\mb{z}_\parallel}=\mb{z},
\end{align}
\par
\textbf{Case 2: $\mb{R}$ triangular, $\mb{\tilde R}$ diagonal}\par
Let us now consider the case when $\mb{R}$ is an upper triangular matrix.
By the argument for the diagonal case above, it suffices to consider the effect of the transformation in $\mathcal{W}$. Multiplying \eqref{eq:orthoflow} by $\mb{Q}^T$ from the left gives:
\begin{align}
\underbrace{\mb{Q}^T\mb{z}'}_{\mb{v}'} = \underbrace{\mb{Q}^T\mb{z}}_{\mb{v}} + \mb{R}h(\mb{\tilde R}\underbrace{\mb{Q}^T\mb{z}}_{\mb{v}}+\mb{b}) \label{eq:triangular}
\end{align}
where the vectors $\mb{v}$ and $\mb{v}'$ contain the respective coordinates of $\mb{z}_\parallel$ and $\mb{z}'_\parallel$ w.r.t. $\mb{q}_1,\ldots,\mb{q}_M$. As in the diagonal case consider the functions $f_i(v) = v + r_{ii}h(\tilde r_{ii} v + b_i)$. Since $\|h'\|_\infty r_{ii} \tilde r_{ii} > -1$, we have $f_i'(v)>0$ and thus $f_i$ is invertible. Let us rewrite \eqref{eq:triangular} in terms of $f_i$:
\begin{align}
v'_{1} &= f_{1}(v_{1}) + \sum_{j=2}^{M}r_{1j}h(\tilde r_{jj} v_j+b_j)\\
\ldots \nonumber\\
v'_{k} &= f_{k}(v_{k}) + \sum_{j=k+1}^{M}r_{kj}h(\tilde r_{jj} v_j+b_j)\\
\ldots\nonumber\\
v'_{M} &= f_{M}(v_{M})
\end{align}
Since $f_M$ is invertible we can write $v_M = f^{-1}_M(v_M')$. Now suppose we have expressed $\{v_j,\forall j>k\}$ in terms of $\{v'_{j}, \forall j > k\}$. Then
\begin{align}
f_{k}(v_{k}) &= v'_{k} - \underbrace{\sum_{j=k+1}^{M}{r_{kj}h(\tilde r_{jj} v_j+b_j)}}_{\text{some function of }\{v'_{j}, \forall j > k\}} \notag \\
&=: g_k(v'_k, v'_{k+1}, \ldots, v'_{M})\\
v_{k} &= f_k^{-1}(g_{k}(v'_k, v'_{k+1}, \ldots, v'_{M})). \notag
\end{align}
Thus we have expressed $\{v_j,\forall j\geq k\}$ in terms of $\{v'_{j}, \forall j \geq k\}$. By induction, we can express $\{v_j,\forall j\}$ in terms of $\{v'_{j}, \forall j\}$ and hence the transformation is invertible. 

\textbf{Case 3: $\mb{R}$ and $\mb{\tilde{R}}$ triangular}\par
Now consider the general case when $\mb{\tilde{R}}$ is triangular. As before we only need to consider the effect of the transformation in $\mathcal{W}$.
\begin{align}
\underbrace{\mb{Q}^T\mb{z}'}_{\mb{v}'} = \underbrace{\mb{Q}^T\mb{z}}_{\mb{v}} + \mb{R}h(\mb{\tilde R}\underbrace{\mb{Q}^T\mb{z}}_{\mb{v}}+\mb{b}) \label{eq:double-triangular}
\end{align}
Let $g$ be the function $g(\mb{v}) = \mb{\tilde R} \mb{v}$. By assumption, $g$ is invertible with inverse $g^{-1}$. Multiplying \eqref{eq:double-triangular} by $\mb{\tilde R}$ gives:
\begin{align}
g(\mb{v'}) = \underbrace{g(\mb{v}) + \mb{\tilde R}\mb{R}h(g(\mb{v})+\mb{b})}_{=:f(g(\mb{v}))}
\end{align}
Since $\mb{\tilde R}\mb{R}$ is upper triangular with diagonal entries $\tilde{r}_{jj}r_{jj}$, $f$ is covered by case 2 considered before and is invertible. Thus, $\mb{v}$ can be written as:
\begin{align}
\mb{v} = g^{-1}(f^{-1}(g(\mb{v'}))). 
\end{align}
Hence the transformation in \eqref{eq:double-triangular} is invertible.
\end{proof}

\subsection{PRESERVING ORTHOGONALITY OF $\mathbf{Q}$}

Orthogonality is a convenient property, mathematically, but hard to achieve in practice. In this paper we consider three different flows based on the theorem above and various ways to preserve the orthogonality of $\mb{Q}$. The first two use explicit differentiable constructions of orthogonal matrices, while the third variant assumes a specific fixed permutation matrix as the orthogonal matrix.

\paragraph{Orthogonal Sylvester flows.} First, we consider a Sylvester flow using matrices with $M$ orthogonal columns (O-SNF). In this flow we can choose $M<D$, and thus introduce a flexible bottleneck. Similar to \citep{hasenclever2017}, we ensure orthogonality of $\mb{Q}$ by applying the following differentiable iterative procedure proposed by \citep{BB:71, K:70}:
\begin{equation}
\mb{Q}^{(k+1)} = \mb{Q}^{(k)} \left( \mb{I} +\frac{1}{2} \left( \mb{I} - \mb{Q}^{(k) \top} \mb{Q}^{(k)} \right) \right) .
\end{equation}
with a sufficient condition for convergence given by $\| \mb{Q}^{(0)\top} \mb{Q}^{(0)} - \mb{I} \|_{2} < 1$. Here, the 2-norm of a matrix $\mb X$ refers to $\|\mb X\|_2 = \lambda_{\mathrm{max}}(\mb X)$, with $\lambda_{\mathrm{max}}(\mb X)$ representing the largest singular value of $\mb X$.
In our experimental evaluations we ran the iterative procedure until $\| \mb{Q}^{(k)\top} \mb{Q}^{(k)} - \mb{I} \|_{F} \leq \epsilon$, with $\|\mb X\|_F$ the Frobenius norm, and $\epsilon$ a small convergence threshold. We observed that running this procedure up to $30$ steps was sufficient to ensure convergence with respect to this threshold. To minimize the computational overhead introduced by orthogonalization we perform this orthogonalization in parallel for all flows.

Since this orthogonalization procedure is differentiable, it allows for the calculation of gradients with respect to $\mb{Q}^{(0)}$ by backpropagation, allowing for any standard optimization scheme such as stochastic gradient descent to be used for updating the flow parameters.

\paragraph{Householder Sylvester flows.} Second, we study Householder Sylvester flows (H-SNF) where the orthogonal matrices are constructed by products of Householder reflections. Householder transformations are reflections about hyperplanes. Let $\mb{v}\in\mathbb{R}^D$, then the reflection about the hyperplane orthogonal to $\mb{v}$ is given by:
\begin{align}
H(\mb{z}) = \mb{z} - 2 \frac{\mb{v}\mb{v}^T}{\|\mb{v}\|^2}\mb{z}
\end{align}

It is worth noting that performing a single Householder transformation is very cheap to compute, as it only requires $D$ parameters. Chaining together several Householder transformations results in more general orthogonal matrices, and it can be shown \citep{hh1, hh2} that any $M \times M$ orthogonal matrix can be written as the product of $M-1$ Householder transformations. In our Householder Sylvester flow, the number of Householder transformations $H$ is a hyperparameter that trades off the number of parameters and the generality of the orthogonal transformation. Note that the use of Householder transformations forces us to use $M=D$, since Householder transformation result in square matrices.

\paragraph{Triangular Sylvester flows.} Third, we consider a triangular Sylvester flow (T-SNF), in which all orthogonal matrices $\mb{Q}$ alternate per transformation between the identity matrix and the permutation matrix corresponding to reversing the order of $\mb z$. This is equivalent to alternating between lower and upper triangular $\mb {\tilde R}$ and $\mb R$ for each flow.

\begin{figure*}[!ht]
\centering
{\includegraphics[width=0.49\textwidth]{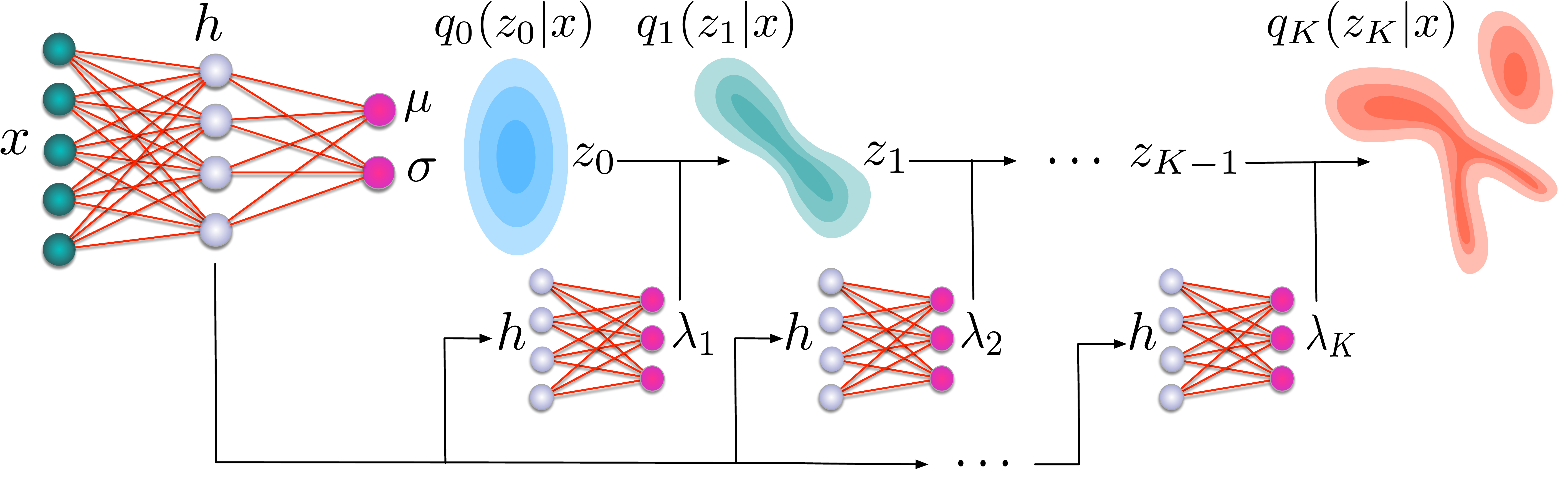}}
\hspace{1mm}
{\includegraphics[width=0.49\textwidth]{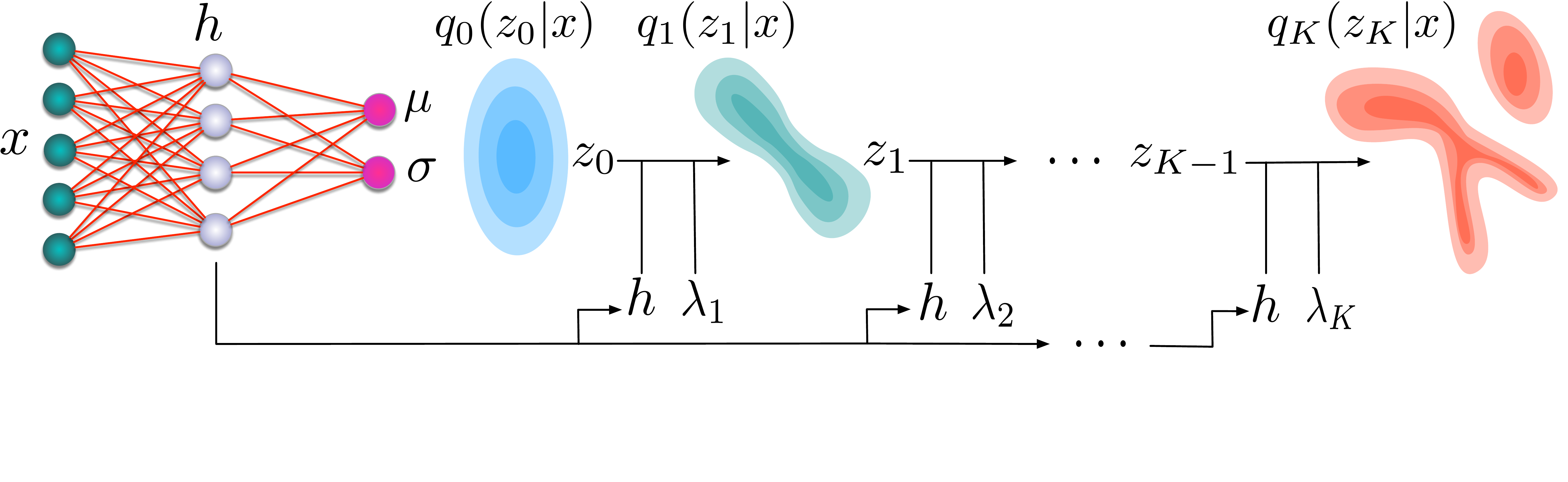}}
\caption{Different amortization strategies for Sylvester normalizing flows and Inverse Autoregressive Flows. Left: our inference network produces input-dependent flow parameters through a hypernetwork \citep{ha2016hypernetworks}. This strategy is also employed by planar flows. Right:  IAF introduces a measure of $\mb{x}$ dependence through a context $\mb{h}(\mb{x})$. This context acts as an additional input for each transformation. The flow parameters themselves are independent of $\mb{x}$, but the number of hidden units per flow is larger than in SNF.}
\label{fig:Amortization}
\end{figure*}

\subsection{DATA-DEPENDENT FLOW PARAMETERS}
As previously mentioned, the parameters of the base distribution as well as the flow parameters can be functions of the data point $\mb{x}$ \citep{RezMoh2015}.  Figure \ref{fig:Amortization} (left) shows a diagram of one SNF step and the amortization procedure. The inference network takes datapoints $\mb x$ as input, and provides as an output the mean and variance of $\mb z^0$ such that $\mb z^0 \sim \mathcal N(\mb z | \pmb \mu^0(\mb x), (\pmb \sigma^{0})^2(\mb x))$. Several SNF transformations are then applied, such that $\mb z^0 \rightarrow \mb z^1\rightarrow \hdots \mb z^K$, producing a flexible posterior distribution for $\mb z^K$. All of the flow parameters ($\mb R$, $\mb {\tilde R}$ and $\mb Q$ for each transformation) are produced as an output of a hypernetwork (attached to the inference network) and are thus functions of $\mb x$.

\section{RELATED WORK}
\label{sec:related_work}
\subsection{NORMALIZING FLOWS FOR VARIATIONAL INFERENCE}

A number of invertible transformations with tractable Jacobians have been proposed in recent years. \citet{RezMoh2015} first discussed such transformations in the context of stochastic variation inference, coining the term normalizing flows. 

\citet{RezMoh2015} proposed two different parametric families of transformations with tractable Jacobians: planar and radial flows. While effective for small problems, these transformations are hard to scale to large latent spaces and often require a large number of transformations. The transformation corresponding to planar flows is given in Eq. \eqref{eq:planar_flow}.

More recently, a successful class of flows called Inverse Autoregressive Flows was introduced in \citep{KinSalJoz2016}. As the name suggests, one IAF transformation can be seen as the inverse of an autoregressive transformation. Consider the following autoregressive transformation:
\begin{align}
&z_0 = \bar{\mu}_0 + \bar\sigma_0 \cdot \epsilon_0 \notag \\
&z_i = \bar\mu_i(\mb z_{1:i-1}) + \bar\sigma_i(\mb z_{1:i-1}) \cdot \epsilon_i, \quad  i=1, \hdots, D
\end{align}
with $\bs \epsilon \sim \mathcal N(\mb 0, \mb I)$. This transformation models the distribution over the variable $\mb z$ with an autoregressive factorization $p(\mb z) = p(z_0) \prod_{i=1}^D p(z_i|z_{i-1}, \hdots, z_0)$. Since the parameters of transformation for $z_i$ are dependent on $\mb z_{1:i-1}$, this procedure requires $D$ sequential steps to sample a single vector $\mb z$. This is undesirable for variational inference, where sampling occurs for every forward pass. 

However, the inverse transformation (which exists if $\bar \sigma_i>0$ $ \forall i$) is easy to sample from:
\begin{align}
\epsilon_i = \frac{z_i - \bar\mu_i(\mb z_{1:i-1})}{\bar \sigma_i(\mb z_{1:i-1})}. 
\end{align}
For this inverse transformation, $\epsilon_i$ is no longer dependent on the transformation of $\epsilon_j$ for $j\neq i$. Hence, this transformation can be computed in parallel: $\bs \epsilon = (\mb z - \bs {\bar\mu} (\mb z))/\bs {\bar\sigma}(\mb z)$. Rewriting $\sigma_i(z_{1:i-1)} = 1/\bar \sigma_i(z_{1:i-1)}$ and  $\mu_i(z_{1:i-1)} = -\bar\mu(z_{1:i-1})/\bar \sigma_i(z_{1:i-1)}$, yields the IAF transformation:
\begin{align}
z^t_i = \mu^t_i(\mb z^{t-1}_{1:i-1}) + \sigma^t_i(\mb z^{t-1}_{1:i-1}) \cdot z^{t-1}_i, \quad i=1,...,D. 
\label{eq:iaf}
\end{align}
Starting from $\mb z^0 \sim \mathcal N(\mb 0, \mb I)$, multiple IAF transformations can be stacked on top of each other to produce flexible probability distributions.

If $\bs \mu^t$ and $\bs \sigma^t$ depend on $\mb z^{t-1}$ linearly, IAF can model full covariance Gaussian distributions. In order to move away from Gaussian distributions to more flexible distributions, it is important that $\bs \mu^t$ and $\bs \sigma^t$ are nonlinear functions of $\mb z^{t-1}$. 

In practice, wide MADEs \citep{Germain15} or deep PixelCNN layers \citep{Oord_pixelcnn_NIPS2016} are needed to increase the flexibility of IAF transformations. This results in transformations with a large number of parameters. As shown in Figure \ref{fig:Amortization} (right), amortization is achieved through a context $\mb{h}(\mb{x})$ that is fed into the autoregressive networks as an additional input at every IAF step. 

Our Triangular Sylvester flows are strongly related to mean-only IAF transformations ($\bs \sigma^t = 1$). As mentioned in \cite{KinSalJoz2016}, between every IAF transformation the order of $\mb z$ is reversed, in order to ensure that on average all dimensions get warped equally. In T-SNF, the same effect is achieved by using the permutation matrix that reverses the order of $\mb z$ in every other transformation as the orthogonal matrix. However, mean-only IAF is a volume-preserving transformation, i.e. the determinant of the Jacobian has absolute value one. T-SNF is not volume preserving due to the nonzero elements on the diagonals of $\mb R$ and $\mb {\tilde{R}}$. Note, that in \cite{KinSalJoz2016} it was shown that the difference in performance between mean-only IAF and the general IAF transformation was negligible.

The most important difference between IAF and T-SNF is the way parameters are amortized. In T-SNF, $\mb R$ and $\mb {\tilde R}$ are directly amortized functions of the input $\mb x$ (see Fig. \ref{fig:Amortization}). This is equivalent to amortizing the MADE parameters in mean-only IAF. Having input dependent MADE parameters allows for flexible transformations with fewer parameters.

Householder Sylvesters flows can also be seen as a non-linear extension of Householder flows \citep{TW:16}. Householder flows are volume-preserving flows, which transform the variational posterior with a diagonal covariance matrix to a full-covariance posterior. Householder flows are a special case of H-SNF if $h(\mb{z}) = \mb{z}$, $\mb{R}$ is the identity matrix, and the residual connection in Eq. \eqref{eq:orthoflow} is left out.

\subsection{NORMALIZING FLOWS FOR DENSITY ESTIMATION}

A number of invertible transformations have been proposed in the context of density estimation. Note that density estimation requires the inverse of the flow to be tractable.
Having a provably invertible transformation is not the same as being able to compute the inverse. 

For density estimation with normalizing flows, we are interested maximizing the log-likelihood of the data:
\begin{align}
\log p(\mb{x})= \log p_{0}(f^{-1}(\mb{x})) + \log \left| \det \left( \frac{\partial f^{-1}(\mb{x})}{\partial \mb{x}} \right) \right|.
\end{align}
Thus, the goal is to transform a complicated data distribution back to a simple distribution. In general, both directions of an invertible transformations need not be tractable. Hence, methods developed for density estimation are generally not directly applicable to variational inference.

Non-linear independent component estimation (NICE, \cite{Dinh2014}) and the related Real NVP \citep{DinSohBen2016}, and Masked Autoregressive Flow (MAF, \cite{2017_Papamakarios_NIPS}) are recent examples of normalizing flows for density estimation. 

In NICE, each transformation splits the variables into two disjoint subsets $\mb{z}_A, \mb{z}_B$. One of the subsets is transformed as $\mb{z}_A'= \mb{z}_A + f(\mb{z}_B)$, while $\mb{z}_B$ is left unchanged. In the next transformation a different subset of variables is transformed. This results in a transformation which is trivially invertible and has a tractable Jacobian. Real NVP uses the same fundamental idea. Appealingly, because of the tractable inverse, NICE and real NVP can generate data and estimate density with one forward pass. However due to fact that only a subset of variables is updated in each transformation many transformations are needed in practice. \cite{RezMoh2015} compared NICE to planar flows in the context of variational inference and found that planar flows empirically perform better. 

Finally, \cite{2017_Papamakarios_NIPS} showed that fitting an MAF can be seen as fitting an implicit IAF from the data distribution to the base distribution. However, generating data from an MAF density model requires $D$ passes, making it unappealing for variational inference. 

\section{NUMBER OF PARAMETERS} 

Here, we briefly compare the number of parameters needed by planar flows, IAF and the three Sylvester normalizing flows. We denote the size of the stochastic variables $z$ with $D$, and the number of output units of the inference network with $E$.

Planar flows use amortized parameters $\mb u, \mb w \in \mathbb R^D$ and $b\in \mathbb{R}$ for each flow transformation. Therefore, the number of parameters related to $K$ flow transformations is equal to $2 EDK + EK$. 

For the implementation of IAF as described in Section \ref{sec:experiments},  the inference network needs to produce a context of size $C$, where $C$ denotes the width of the MADE layers. The total number of flow related learnable parameters then comes down to $EC + K \times (C^2 + 3CD)$.

In the case of Orthogonal Sylvester flows with a bottleneck of size $M$, we require $ KE \times(MD + 2M^2 + M)$ parameters. For Householder Sylvester flows with $H$ Householder reflections per flow transformation, $KE\times(HD + 2D^2 + D)$ parameters are needed. Finally, for triangular Sylvester flows $KE \times(2D^2 + D)$ parameters require optimization.

Planar flows require the smallest number of parameters but generally result in worse results. IAFs on the other hand require a number of parameters that is quadratic in the width of the MADE layers. For good results this has to be quite large. In contrast, for SNFs the number of parameters is quadratic in the dimension of the latent space and while large, this can still be amortized.

\section{EXPERIMENTS}
\label{sec:experiments}

We perform empirical studies of the performance of Sylvester flows on four datasets: statically binarized MNIST, Freyfaces, Omniglot and Caltech 101 Silhouettes. The baseline model is a plain VAE with a fully factorized Gaussian distribution. We furthermore compare against planar flows and Inverse Autoregressive Flows of different sizes.

We use annealing to optimize the lower bound, where the prefactor of the KL divergence is linearly increased from 0 to 1 during 100 epochs as suggested by \cite{Bowman2015} and \cite{Sonderby2016}. A learning rate of $0.0005$ was used in all experiments. In order to obtain estimates for the negative log likelihood we used importance sampling (as proposed in \citep{RMW:14}). Unless otherwise stated, 5000 importance samples were used.

In order to assess the performance of the different flows properly, we use the same base encoder and decoder architecture for all models.
We use gated convolutions and transposed convolutions as base layers for the encoder and decoder architecture respectively. The inference network consists of several gated convolution layers that produce a hidden unit vector. After being flattened, these hidden units act as an input to two fully connected layers that predict the mean and variance of $\mb z^0$. 

For planar and Sylvester flows, the flattened hidden units are passed to a separate linear layer that output the amortized flow parameters. For IAF, the flattened hidden units are also passed to a linear layer to produce the context vector $\mb h_{\mathrm{context}}(\mb x)$.
For details of the architecture see Section \ref{sec:app_architecture} of the appendix. In all models the latent space is of dimension $64$.

\begin{figure}[ht]
\begin{center}
\centerline{\includegraphics[width=1.05\columnwidth]{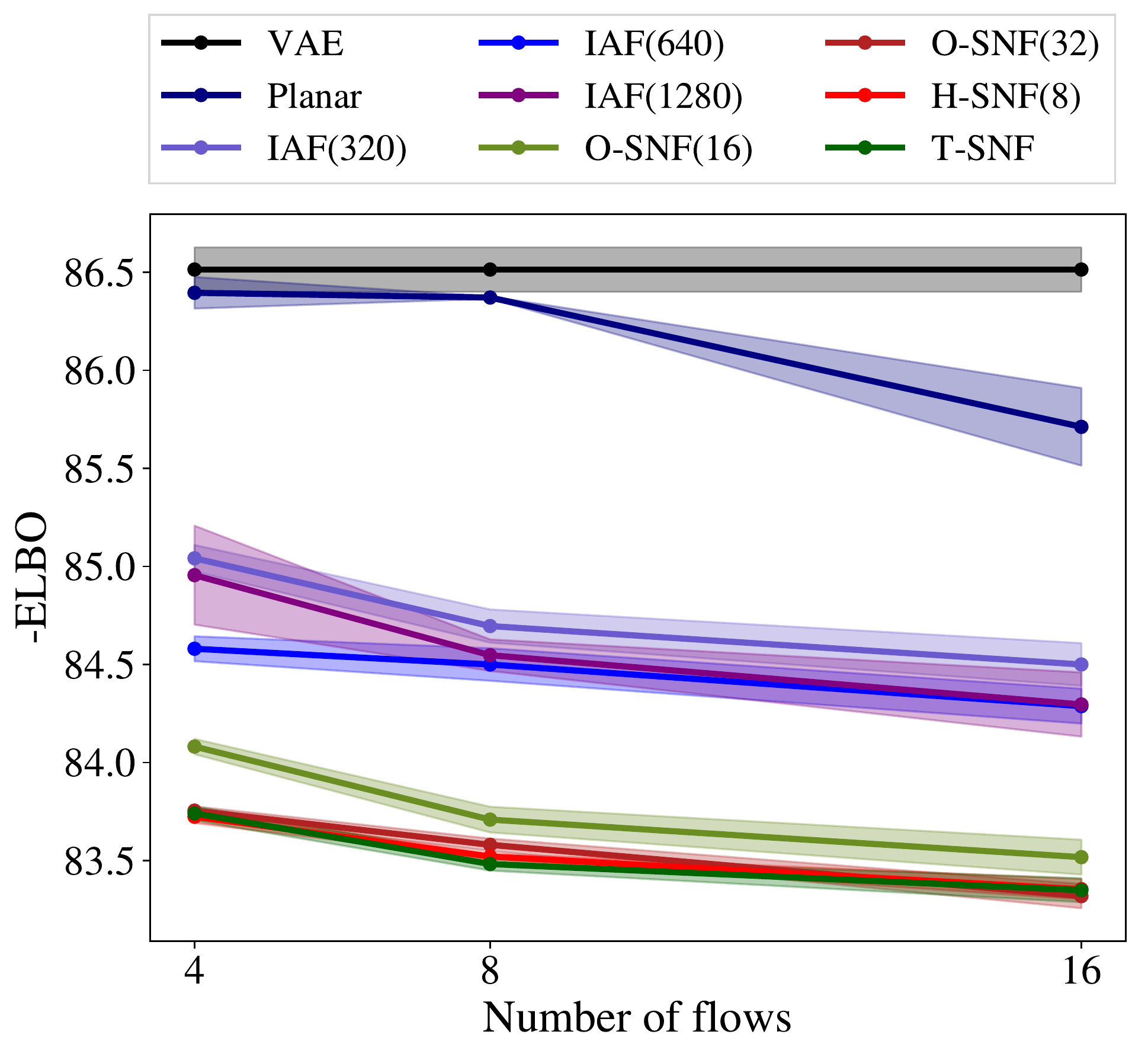}}
\vskip -3mm
  \caption{The negative evidence lower bound for static MNIST. The results for H-SNF with 4 reflections per orthogonal matrix are left out for clarity, as they are very similar to the results with 8 reflections. Each model is evaluated 3 times. The shaded areas indicate $\pm$ one standard deviation.}
  \label{fig:results_mnist}
  \end{center}
\vskip -0.2in
 \end{figure}

We use the following implementation for each IAF transformation\footnote{This implementation is based on the open source code for IAF available at \url{https://github.com/openai/iaf}}: one IAF transformation first applies one MADE Layer (denoted as MaskedLinear) followed by a nonlinearity to the input $z$, upscaling it to a hidden variable of size $M$. At this point the context vector $\mb h_{\mathrm{context}}(\mb x)$ is added to the hidden units, after which two more masked layers are applied to produce the mean and scale of the IAF transformation:
\begin{align}
&\mb h_z \leftarrow \mathrm{ELU}(\mathrm{MaskedLinear}(\mb z)) \notag \\
&\mb h \leftarrow \mb h_z + \mb h_{\mathrm{context}}(\mb x) \notag\\
&\mb h \leftarrow \mathrm{ELU}(\mathrm{MaskedLinear}(\mb h)) \notag \\
&\bs \mu \leftarrow \mathrm{MaskedLinear}(\mb h), \quad\mb s \leftarrow \mathrm{MaskedLinear}(\mb h) \notag \\
& \mb z' \leftarrow \sigma(\mb s) \odot \mb z + (1-\sigma(\mb s)) \odot \bs \mu.
\label{eq:iaf_setup}
\end{align}
Here, $\sigma(\;)$ denotes the sigmoid activation function. In \cite{KinSalJoz2016} it was mentioned that the gated form of IAF in Eq. \eqref{eq:iaf_setup} is more stable than the form of Eq. \eqref{eq:iaf}. Note that the size of $\mb h_{\mathrm{context}} (\mb x)$ scales with the width of the MADE layers $C$. 

\begin{table}[htb]
\centering
\caption{Negative log-likelihood and free energy (negative evidence lower bound) for static MNIST. Numbers are produced with 3 runs per model with different random initializations. Standard deviations over the 3 different runs are also shown.
\vspace{0.3cm}}
\begin{tabular}{l r r}
\toprule
\textbf{Model} & \textbf{-ELBO} & \textbf{NLL} \\
[0.05em]\midrule \\[-0.8em]
VAE & $86.55 \pm 0.06$ & $82.14 \pm 0.07$ \\
Planar & $86.06 \pm 0.31$ & $81.91 \pm 0.22$\\
IAF & $84.20 \pm 0.17$ & $80.79 \pm 0.12$ \\ \hdashline
O-SNF & $\mathbf{83.32 \pm 0.06}$ & $\mathbf{80.22 \pm 0.03}$ \\
H-SNF & $83.40 \pm 0.01$ & $80.29 \pm 0.02$\\
T-SNF & $83.40 \pm 0.10$ & $80.28 \pm 0.06$\\ \bottomrule
\end{tabular}
\label{tab:NLL_mnist} 
\end{table}

\begin{table*}[ht]
\centering
\caption{Results for Freyfaces, Omniglot and Caltech 101 Silhouettes datasets. For the Freyfaces dataset the results are reported in bits per dim. For the other datasets the results are reported in nats. For each flow model 16 flows are used. For IAF a MADE width of 1280 was used, and for O-SNF flow a bottleneck of $M=32$ was used. For H-SNF $8$ householder reflections were used to construct orthogonal matrices. For all datasets 3 runs per model were performed.}
\vspace*{0.3cm}
\begin{tabular}{l  r r r r r r}
\toprule
\multirow{2}{*}{\bfseries{Model}} & 
 \multicolumn{2}{c}{\bfseries Freyfaces} &
 \multicolumn{2}{c}{\bfseries Omniglot} &
 \multicolumn{2}{c}{\bfseries Caltech 101}\\ 
& \multicolumn{1}{c}{-ELBO} & \multicolumn{1}{c}{NLL} & \multicolumn{1}{c}{-ELBO} & \multicolumn{1}{c}{NLL} & \multicolumn{1}{c}{-ELBO} &\multicolumn{1}{c}{NLL} \\\cmidrule(lr){1-7}
VAE & $4.53 \pm 0.02$ & $4.40 \pm 0.03$ & $104.28 \pm 0.39$ & $97.25 \pm 0.23$ 
& $110.80 \pm 0.46$ & $99.62 \pm 0.74$  \\
Planar & $\mathbf{4.40 \pm 0.06}$ & $\mathbf{4.31 \pm 0.06}$ & $102.65 \pm 0.42$ & $96.04 \pm 0.28$ 
& $109.66 \pm 0.42$ & $98.53 \pm 0.68$\\
IAF & $4.47 \pm 0.05$ & $4.38 \pm 0.04$ & $102.41 \pm 0.04$ & $ 96.08 \pm 0.16$ & 
$111.58 \pm 0.38$ & $ 99.92 \pm 0.30$\\\hdashline
O-SNF & $4.51 \pm 0.04$ & $4.39 \pm 0.05$ & $99.00 \pm 0.29$ & $93.82 \pm 0.21$
& $106.08 \pm 0.39$ & $94.61 \pm 0.83$\\
H-SNF & $4.46 \pm 0.05$ & $4.35 \pm 0.05$& $\mathbf{99.00 \pm 0.04}$ & $\mathbf{93.77 \pm 0.03}$ & $\mathbf{104.62 \pm 0.29}$ & $\mathbf{93.82 \pm 0.62}$\\
T-SNF & $4.45 \pm 0.04$ & $4.35 \pm 0.04$ & $99.33 \pm 0.23$ & $93.97 \pm 0.13$
& $105.29 \pm 0.64$ & $94.92 \pm 0.73$\\
\bottomrule
\end{tabular}
\label{tab:results_freyfaces} 
\end{table*}

\subsection{MNIST}
 
Figure \ref{fig:results_mnist} shows the dependence of the negative evidence lower bound (or free energy) on the number of flows and the type of flow for static MNIST. The exact numbers corresponding to the figure are shown in Section \ref{sec:app_mnist} in the appendix. 

For all models the performance improves as a functions of the number of flows. For 4 flows the difference between the baseline VAE and planar flows is very small. However, planar flows clearly benefit from more flow transformations.

For IAF three different widths of the MADE layers were used: $C=320$, 640 and 1280. Surprisingly, for 4 flows the widest IAF with 1280 hidden units is outperformed by an IAF with 640 hidden units in the MADE layers. We expect this to be due to the fact that this model has more parameters and can therefore be harder to train, as indicated by the larger standard deviation for this model.

All three Sylvester flows outperform IAF and planar flows. For Orthogonal Sylvester flows, we show results for $M=16$ and $M=32$ orthogonal vectors per orthogonal matrix, thus corresponding to bottlenecks of size 16 and 32 respectively for a latent space of size $D=64$. Clearly, a larger bottleneck improves performance. 
For Householder Sylvester flows we experimented with $H=4$ and $H=8$ Householder reflections per orthogonal matrix. Since the results were nearly indistinguishable between these two variants, we have left out the curve for $H=4$ to avoid clutter. O-SNF with $M=32$, H-SNF and T-SNF seem to perform on par. 

In Table \ref{tab:NLL_mnist}, the negative evidence lower bound and the estimated negative log-likelihood are shown for the baseline VAE, together with all flow models for 16 flows. The reported result for IAF is for a MADE width of 1280. The O-SNF model has a bottleneck of $M=32$, and H-SNF contains 8 Householder reflections per orthogonal matrix. Again, all Sylvester flows outperform planar flows and IAF, both in terms of the free energy and the negative log-likelihood. 

As discussed in Section \ref{sec:related_work}, T-SNF is closely related to mean-only IAF, but with the MADE parameters produced by a hypernetwork that depends on the input data $\mb x$. The fact that T-SNF outperforms IAF indicates that having data-dependent flow parameters directly leads to a more flexible transformation compared to taking a very wide MADE with a data-dependent context as an additional input.
 
\subsection{FREYFACES, OMNIGLOT AND CALTECH 101 SILHOUETTES}

We further assess the performance of the different models on Freyfaces, Omniglot and Caltech 101 Silhouettes. The results are shown in Table \ref{tab:results_freyfaces}. The model settings are the same\footnote{For Caltech 101 Silhouettes we used 2000 importance samples for the estimation of the negative log-likelihood.}
 as those used for Table \ref{tab:NLL_mnist}.
 
Freyfaces is a very small dataset of around 2000 faces. All normalizing flows increase the performance, with planar flows yielding the best result, closely followed by Triangular and Householder Sylvester flows. We expect planar flows to perform the best in this case since it is the least sensitive to overfitting.


For Omniglot and Caltech 101 Silhouettes the results are clearer, with  the Sylvester normalizing flows family resulting in the best performance. Both H-SNF and T-SNF perform better than O-SNF. This could be attributed to the fact that O-SNF has a bottleneck of $M=32$ for a latent space size of $D=64$. The IAF scores for Caltech 101 are surprisingly bad. We expect this could be the case due to the large number of parameters that need to be trained for IAF(1280). Therefore we also evaluated the result for MADEs of width 320 for 16 flows. The resulting free energy and estimated negative log-likelihood are $111.23 \pm 0.45$ and $99.74 \pm 0.28$ respectively, only slightly improving on the results of 1280 wide IAFs. 

\section{CONCLUSION}

We present a new family of normalizing flows: Sylvester normalizing flows. These flows generalize planar flows, while maintaining an efficiently computable Jacobian determinant through the use of Sylvester's determinant identity. We ensure invertibility of the flows through the use of orthogonal and triangular parameter matrices. Three variants of Sylvester flows are investigated. First, orthogonal Sylvester flows use an iterative procedure to maintain orthogonality of parameter matrices. Second, Householder Sylvester flows use Householder reflections to construct orthogonal matrices. Third, triangular Sylvester flows alternate between fixed permutation and identity matrices for the orthogonal matrices. We show that the triangular Sylvester flows are closely related to mean-only IAF, with data-dependent MADE parameters. While performing comparably with planar flows and IAF for the Freyfaces dataset, our proposed family of flows improve significantly upon planar flows and IAF on the three other datasets. 

\subsubsection*{Acknowledgements}
We would like to thank Christos Louizos for helping with the implementation of inverse autoregressive flows, and Diederik Kingma for fruitful discussions.
LH is funded by the UK EPSRC OxWaSP CDT through grant EP/L016710/1. JMT is funded by the European Commission within the MSC-IF (Grant No. 702666). RvdB is funded by SAP SE.


\clearpage
\appendix

\section{Architecture}
\label{sec:app_architecture}
In the experiments we used convolutional layers for both the encoder and the decoder. Moreover, we used the gated activation function for convolutional layers:
\begin{equation*}
\mathbf{h}_{l} = ( \mathbf{W}_{l} * \mathbf{h}_{l-1} + \mathbf{b}_{l} ) \odot \sigma ( \mathbf{V}_{l} * \mathbf{h}_{l-1} + \mathbf{c}_{l} ) ,
\end{equation*}
where $\mathbf{h}_{l-1}$ and $\mathbf{h}_{l}$ are inputs and outputs of the $l$-th layer, respectively, $\mathbf{W}_{l}, \mathbf{V}_{l}$ are weights of the $l$-th layer, $\mathbf{b}_{l}, \mathbf{c}_{l}$ denote biases, $*$ is the convolution operator, and $\sigma(\cdot)$ is the sigmoid activation function.

We used the following architecture of the encoder ($\mathrm{k}$ is a kernel size, $\mathrm{p}$ is a padding size, and $\mathrm{s}$ is a stride size):\footnote{We use a PyTorch convention of defining convolutional layers.}
\begin{align*}
& \mathrm{Conv(in=1, out=32, k=5, p=2, s=1)} \notag \\
& \mathrm{Conv(in=32, out=32, k=5, p=2, s=2)} \notag \\
& \mathrm{Conv(in=32, out=64, k=5, p=2, s=1)} \notag \\
& \mathrm{Conv(in=64, out=64, k=5, p=2, s=2)} \notag \\
& \mathrm{Conv(in=64, out=64, k=5, p=2, s=1)} \notag \\
& \mathrm{Conv(in=64, out=64, k=5, p=2, s=1)} \notag \\
& \mathrm{Conv(in=64, out=256, k=7, p=0, s=1)} \notag 
\end{align*}
Notice the last layer acts as a fully-connected layer. Eventually, fully-connected linear layers were used to parameterized diagonal Gaussian distribution and amortized parameters of a flow.

The decoder mirrors the structure of the encoder with transposed convolutional layers ($\mathrm{op}$ is an outer padding):
\begin{align*}
& \mathrm{ConvT(in=64, out=64, k=7, p=0, s=1)} \notag \\
& \mathrm{ConvT(in=64, out=64, k=5, p=2, s=1)} \notag \\
& \mathrm{ConvT(in=64, out=32, k=5, p=2, s=2, op=1)} \notag \\
& \mathrm{ConvT(in=32, out=32, k=5, p=2, s=1)} \notag \\
& \mathrm{ConvT(in=32, out=32, k=5, p=2, s=2, op=1)} \notag \\
& \mathrm{ConvT(in=32, out=32, k=5, p=2, s=1)} \notag \\
& \mathrm{ConvT(in=32, out=1, k=1, p=0, s=1)} \notag 
\end{align*}

\subsection{Description of datasets}

In the experimetns we used the following four image datasets: static MNIST\footnote{\url{http://yann.lecun.com/exdb/mnist/}}, OMNIGLOT\footnote{\url{https://github.com/yburda/iwae/blob/master/datasets/OMNIGLOT/chardata.mat}.}, Caltech 101 Silhouettes\footnote{\url{https://people.cs.umass.edu/~marlin/data/caltech101_silhouettes_28_split1.mat}.}, and Frey Faces\footnote{\url{http://www.cs.nyu.edu/~roweis/data/frey_rawface.mat}}. Frey Faces contains images of size $28 \times 20$ and all other datasets contain $28 \times 28$ images.

MNIST consists of hand-written digits split into 60,000 training datapoints and 10,000 test sample points. In order to perform model selection we put aside 10,000 images from the training set.

OMNIGLOT is a dataset containing 1,623 hand-written characters from 50 various alphabets. Each character is represented by about 20 images that makes the problem very challenging. The dataset is split into 24,345 training datapoints and 8,070 test images. We randomly pick 1,345 training examples for validation. During training we applied dynamic binarization of data similarly to dynamic MNIST.

Caltech 101 Silhouettes contains images representing silhouettes of 101 object classes. Each image is a filled, black polygon of an object on a white background. There are 4,100 training images, 2,264 validation datapoints and 2,307 test examples. The dataset is characterized by a small training sample size and many classes that makes the learning problem ambitious.

Frey Faces is a dataset of faces of a one person with different emotional expressions. The dataset consists of nearly 2,000 gray-scaled images. We randomly split them into 1,565 training images, $200$ validation images and $200$ test images. We repeated the experiment $3$ times.

\section{MNIST experiments}
\label{sec:app_mnist}
The exact numbers for the evidence lower bound as shown in Fig. \ref{fig:results_mnist} are listed in Table \ref{tab:results_mnist}.
\begin{table}[h]
\centering
\caption{Negative evidence lower bounds for the test set on MNIST. All results are obtained with stochastic hidden units of size $64$.}
\vskip 3mm
\begin{tabular}{l r }
\toprule
\textbf{Model} & \textbf{-ELBO}\\
[0.05em]\midrule 
VAE & $86.51 \pm 0.11$  \\
\hdashline
Planar ($K=4$) & $86.40 \pm 0.08$ \\
Planar ($K=8$) & $86.37 \pm 0.006$ \\
Planar ($K=16$) & $85.71 \pm 0.20$ \\
\hdashline
IAF ($W=320$, $K=4$) & $85.04 \pm 0.07$ \\
IAF ($W=320$, $K=8$) & $84.70 \pm 0.08$ \\
IAF ($W=320$, $K=16$) & $84.50 \pm 0.11$ \\
\hdashline
IAF ($W=640$, $K=4$) & $84.58 \pm 0.06$ \\
IAF ($W=640$, $K=8$) & $84.50 \pm 0.08$ \\
IAF ($W=640$, $K=16$) & $84.29 \pm 0.09$ \\
\hdashline
IAF ($W=1280$, $K=4$) & $84.96 \pm 0.25$ \\
IAF ($W=1280$, $K=8$) & $84.55 \pm 0.08$ \\
IAF ($W=1280$, $K=16$) & $84.30 \pm 0.16$ \\
\hdashline
O-SNF ($M=16$, $K=4$) & $84.08 \pm 0.04$ \\
O-SNF ($M=16$, $K=8$) & $83.71 \pm 0.07$ \\
O-SNF ($M=16$, $K=16$) & $83.52 \pm 0.09$ \\
\hdashline
O-SNF ($M=32$, $K=4$) & $83.76 \pm 0.02$ \\
O-SNF ($M=32$, $K=8$) & $83.58 \pm 0.03$ \\
O-SNF ($M=32$, $K=16$) & $83.32 \pm 0.06$ \\
\hdashline
H-SNF ($K=4$, $H=4$) & $83.73 \pm 0.05$ \\
H-SNF ($K=8$, $H=4$) & $83.49 \pm 0.08$ \\
H-SNF ($K=16$, $H=4$) & $83.36 \pm 0.04$ \\
\hdashline
H-SNF ($K=4$, $H=8$) & $83.72 \pm 0.03$ \\
H-SNF ($K=8$, $H=8$) & $83.52 \pm 0.01$ \\
H-SNF ($K=16$, $H=8$) & $83.35 \pm 0.05$ \\
\hdashline
T-SNF ($K=4$) & $83.74 \pm 0.04$ \\
T-SNF ($K=8$) & $83.48 \pm 0.03$ \\
T-SNF ($K=16$) & $83.35 \pm 0.06$ \\
\bottomrule
\end{tabular}
\label{tab:results_mnist} 
\end{table}

\end{document}